\documentclass[journal]{IEEEtran}
\usepackage[utf8]{inputenc}
\usepackage{graphicx}

\usepackage{booktabs}

\usepackage{amsmath} 
\usepackage{bbm}
\usepackage{bm}
\usepackage{amssymb}
\usepackage{mathtools}
\let\theoremstyle\relax
\usepackage{amsthm}

\theoremstyle{definition}
\newtheorem{definition}{Definition}
\newtheorem{theorem}{Theorem}
\newtheorem{lemma}{Lemma}
\newtheorem{assumption}{Assumption}
\newtheorem{corollary}{Corollary}
\newtheorem{problem}{Problem}

\newtheorem{remark}{Remark}

\usepackage{hyperref}
\hypersetup{
    colorlinks=true,
    linkcolor=black,
    filecolor=black,      
    urlcolor=blue,
    citecolor=blue,
    }
\usepackage[sort&compress,square,numbers]{natbib}
\usepackage{todonotes}

\usepackage{blindtext} 

\begin{document}

\title{A Constraint-Driven Approach to Line Flocking: The V Formation as an Energy-Saving Strategy}
\author{Logan E. Beaver,
\IEEEmembership{Member, IEEE},
Christopher Kroninger, Michael Dorothy,\\
Andreas A. Malikopoulos,
\IEEEmembership{Senior Member, IEEE}
\thanks{L.E. Beaver is with the Division of Systems Engineering, Boston University, Boston, MA, USA (email: \tt\small{lebeaver@bu.edu})}
\thanks{C. Kroninger and M. Dorothy are with Combat Capabilities Development Command, Army Research Laboratory, MD, USA (email(s): \tt\small{christopher.m.kroninger.civ@army.mil, michael.r.dorothy.civ@army.mil})}
\thanks{A.A. Malikopoulos is with the Department of Mechanical Engineering, University of Delaware, Newark, DE, USA (email: \tt\small{andreas@udel.edu})}
}

\markboth{IEEE Transactions on Robotics,
Vol. 1, No. 1, January 2022}
{Beaver
\MakeLowercase{\textit{(et al.)}:
V Formation}
}

\maketitle

\begin{abstract}

The study of robotic flocking has received significant attention in the past twenty years.
In this article, we present a constraint-driven control algorithm that minimizes the energy consumption of individual agents and yields an emergent V formation.
As the formation emerges from the decentralized interaction between agents, our approach is robust to the spontaneous addition or removal of agents to the system.
First, we present an analytical model for the trailing upwash behind a fixed-wing UAV, and we derive the optimal air speed for trailing UAVs to maximize their travel endurance.
Next, we prove that simply flying at the optimal airspeed will never lead to emergent flocking behavior, and we propose a new decentralized ``anseroid" behavior that yields emergent V formations.
We encode these behaviors in a constraint-driven control algorithm that minimizes the locomotive power of each UAV.
Finally, we prove that UAVs initialized in an approximate V or echelon formation will converge under our proposed control law, and we demonstrate this emergence occurs in real-time in simulation and in physical experiments with a fleet of Crazyflie quadrotors.
\end{abstract}


\section{Introduction}

Multi-agent systems have attracted considerable attention in many applications due to their natural parallelization, general adaptability, and ability to self-organize \cite{Oh2017}. Such systems have a nonclassical information structure \cite{Malikopoulos2021} and impose several challenges in learning and control \cite{Malikopoulos2022a}.
A recent push in constraint-driven control has brought the idea of long-duration autonomy to the forefront of multi-agent systems research \cite{Egerstedt2018RobotAutonomy}.
As the real-world deployment of robotic systems continues \cite{mahbub2020sae-1,Vasarhelyi2018OptimizedEnvironments} these robotic agents will be left to to interact with their environment on timescales significantly longer than what can be achieved in a laboratory setting.
In addition, constructing a swarm with a large number of robots imposes significant cost constraints on each individual.
For robotic swarms to become viable, we need to develop and employ energy-optimal control techniques under strict cost constraints.
Thus, any approach to long-duration autonomy must emphasize safe energy-minimizing control policies that are driven by interactions with the environment.

Generating emergent flocking behavior has been of particular interest since Reynolds proposed three heuristic rules for multi-agent flocking in computer animation \cite{Reynolds1987}.
In aerial systems, the main energy savings comes from upwash, i.e., trailing regions of upward momentum in the slipstream, which followers exploit to reduce induced drag. 
Flocking to minimize energy consumption is known as line flocking in the engineering literature \cite{Beaver2020AnFlockingb}, and its name comes from the linear formation-like flocking behavior of geese, pelicans, etc \cite{Bajec2009OrganizedBirds}.

The simplest method to achieve a V formation may be to generate an optimal set of formation points based on the aerodynamic characteristics of each agent.
This effectively transforms the line flocking problem into a formation reconfiguration problem, where each agent must assign itself to a unique goal and reach it within some fixed time \cite{Nathan2008V-likeBirds}. 
The physical effects of V formation flight were explored in recent article \cite{Mirzaeinia2019EnergyReconfiguration}, where the authors demonstrate that the leading and trailing agents consume energy at a significantly higher rate.
This implies that these agents are the limiting factor in the total distance traveled, and the authors propose a formation-reconfiguration scheme based on a load-balancing protocol.
However, formation reconfiguration approaches generally require the formation to be computed offline, and while some  articles consider agent heterogeneity (e.g., age, weight, size, and efficiency) \cite{Mirzaeinia2020AnPurposes}, this has not yet been explicitly considered in an engineering context.
Furthermore, formation points must be recalculated online if an agent enters or leaves the system, or if there are significant changes in the ambient environment.

A more flexible approach is to treat line flocking as a data-driven problem, where agents measure the local aerodynamic and hydrodynamic interactions to dynamically position themselves in the upwash field and save energy.
This has been achieved for aerial vehicles in $\mathbb{R}^2$ using a model predictive control approach \cite{Yang2016LoveControl}, where the authors construct a multi-objective optimization problem that minimizes speed differences, maximizes upwash benefit, and minimizes the field of view occlusion between agents.
The authors also demonstrate that solving this multi-objective optimization problem yields emergent V formations, even when the agents are initialized randomly.
A recent review of related optimal flocking techniques is presented in \cite{Beaver2020AnFlockingb}.

Our approach, in contrast to existing work, is constraint-driven.
In our framework, agents seek to travel as efficiently as possible subject to a set of task and safety constraints.
This set-theoretic approach to control is interpretable, i.e., the cause of an agent's action can be deduced by examining which constraints become active during operation.
By examining the conditions that lead to an empty feasible space, our framework also addresses the problem of constraint compatibility, i.e., how each agent ought to behave when its feasible action space is over-constrained.
Our approach is decentralized, and thus it is well-suited to ``open systems," where agents may suddenly enter, leave, or experience failure.

In this article, we describe a new \textit{anseroids} behavior (anserine-oid, meaning `goose-like') that generates dynamic echelon and V formations without any knowledge of the total number of agents in the system, and which are not given any information about the desired formation shape.
To the best of our knowledge, the only results similar to ours are these reported in \cite{Yang2016LoveControl,Roy2020LearningV-formation}.
The former uses particle swarm optimization combined with model predictive control to solve an optimal control problem. However, their approach depends on a multi-objective optimization problem with four components, and they provide no guarantees on the emergence of flocking behavior. 
The latter employs reinforcement learning to generate conditions that ensure a V formations does not occur.
The contributions of this article are as follows:
\begin{itemize}
    \item The first optimal control algorithm that demonstrates emergent V formations as a result of an energy-minimizing control policy,
    \item a physics-based flocking model where agent decisions are coupled through aerodynamic interactions,
    \item an interpretable set-theoretic control architecture that intuitively describes the optimal behavior of each agent in the flock, 
    \item a switching control policy that guarantees a solution to the constrained optimal control problem always exists, and
    \item compelling evidence that demonstrate how energy savings is enhanced when heterogeneity is introduced to the flock.
\end{itemize}

The remainder of the article is organized as follows.
In Section \ref{sec:problem}, we discuss our notation and present the dynamics of our problem.
We present our optimal control problem and guarantees on its behavior in Section \ref{sec:control}, and in Section \ref{sec:validation} we validate our results in simulation and experimentally.
In particular, we simulate $2$ UAVs in Section \ref{sec:matlab} to show the aerodynamic interactions between agents, and in Section \ref{sec:experiment} we simulate $11$ UAVs in CrazySwarm and validate our controller experimentally with $5$ drones.
Finally, we summarize our results and directions for future research in Section \ref{sec:conclusion}.

\section{Problem Formulation} \label{sec:problem}

\subsection{Note on Notation}

Most references on optimal control, e.g., \cite{Bryson1975AppliedControl,Ross2015}, consider centralized problems.
Thus, directly adopting their notation may lead to ambiguities about the state space of a decentralized problem.
To relieve this tension, we take the following approach for an agent with index $i$:
endogenous variables, e.g., the position of agent $i$, are written without an explicit dependence on time, while exogenous variables, e.g., the position of agent $j$ as measured by agent $i$, are written with an explicit dependence on time.
This notation is common in the applied mathematics literature \cite{Levine2011OnFlatness}, and makes it explicitly clear how functions evolve with respect to the state (e.g., state dynamics) and how they evolve with respect to time (e.g., external signals measured by the agent).

\subsection{System Dynamics}

We consider a fleet of $N\in\mathbb{N}, N \geq 3$ fixed-wing uncrewed aerial vehicles (UAVs) indexed by the set $\mathcal{A} = \{1, 2, \dots, N\}$.
We denote the state of each UAV $i\in\mathcal{A}$ by
\begin{equation}
    \bm{x}_i \coloneqq 
    \begin{bmatrix}
    \bm{p}_i \\ \theta_i
    \end{bmatrix},
\end{equation}
where $\bm{p}_i\in\mathbb{R}^2$ is the UAV's position and $\theta_i\in\mathbb{R}$ is the UAV's heading angle.
Each UAV obeys unicycle dynamics,
\begin{equation} \label{eq:dynamics}
\begin{aligned}
    \dot{\bm{p}}_i &= 
    \begin{bmatrix}
    v_i \cos{\theta} \\
    v_i \sin{\theta}
    \end{bmatrix}, \\
    \dot{\theta}_i &= \omega_i,
\end{aligned}
\end{equation}
where $v_i\in\mathbb{R}_{>0}$ and $\omega_i\in\mathbb{R}$ are the control variables, which correspond to the linear and angular speed of UAV $i$.
We impose the control bounds,
\begin{equation}
\begin{aligned} \label{eq:bounds}
    |\omega_i| &\leq \omega_{\max}, \\
    0 < v_{\min} \leq v_i &\leq v_{\max},
\end{aligned}
\end{equation}
where $\omega_{\max}$ is the maximum turning rate and $v_{\min} < v_{\max}$ correspond to the minimum and maximum air speed.

Finally, the total drag force acting on UAV $i$ has the form,
\begin{align} \label{eq:drag}
    F_i(\bm{x}_i, t) &= C_1 v_i^2 + \frac{C_2}{v_i^2} - \frac{L}{v_i} W(\bm{p}_i, t),
\end{align}
where $C_1, C_2 \in\mathbb{R}_{>0}$ capture the profile and self-induced drag that include the drag coefficient, air density, and wing area.
The function $W : \mathbb{R}^2\times\mathbb{R} \to \mathbb{R}$ describes the scalar upwash field, which we formally define in Section \ref{sec:control}.

In our modeling framework above, we impose the following assumptions.

\begin{assumption} \label{smp:tracking}
    Each UAV is equipped with a low-level flight controller that can track the sequence of control actions that we generate.
\end{assumption}
\begin{assumption} \label{smp:stillAir}
    The UAVs are operating in still air with constant aerodynamic properties.
\end{assumption}

We impose Assumptions \ref{smp:tracking} and \ref{smp:stillAir} to determine the behavior of the system in idealized conditions. However, these assumptions are not restrictive on our analysis.
In particular, applying adaptive and robust control techniques, such as control barrier functions \cite{Wang2017SafetySystems} or Gaussian Processes \cite{chalaki2021RobustGP}, can be used to overcome the resulting model mismatch.

\begin{assumption} \label{smp:collision}
    Collision avoidance between UAVs in $\mathcal{A}$ can be neglected.
\end{assumption}

Generally, V formations have significant space between individuals without opportunities for collisions between agents \cite{Mirzaeinia2019EnergyReconfiguration,Mirzaeinia2020AnPurposes}.
Thus, we impose Assumption \ref{smp:collision} to focus the scope of our work on the emergence of the V formation.
Furthermore, aerodynamic effects disincentive UAVs from approaching too closely, and collision avoidance can always be guaranteed by introducing pairwise collision avoidance constraints \cite{Beaver2020Energy-OptimalConstraints} or employing a control barrier function as a safety layer \cite{Wang2017SafetySystems}.

\begin{assumption} \label{smp:main}
    There exists a global heading angle $\theta_g$ and a small tolerance $\epsilon\in\mathbb{R}_{>0}$ such that $|\theta_i - \theta_g| < \epsilon$ for all UAVs $i\in\mathcal{A}$.
\end{assumption}

We impose Assumption \ref{smp:main} to simplify our analysis of the aerodynamics.
First, it allows us to consider spanwise cuts of the domain, which reduces our analysis from 2D to 1D.
Second, it allows us to model the wake as a scalar field centered on each UAV instead of modeling the wake evolution numerically, i.e., using computational fluid dynamics.
This assumption is common in the multi-UAV literature \cite{Yang2016LoveControl,Mirzaeinia2019EnergyReconfiguration}, although it is usually not stated explicitly.
We impose this assumption as a constraint in our final control algorithm, and it can be interpreted as the ``migratory urge'' proposed by Reynolds \cite{Reynolds1987}; the direction $\theta_g$ could also be computed using consensus, and some agents could simply separate themselves from the flock if Assumption \ref{smp:main} becomes too restrictive.

\subsection{Wake Model} \label{sec:aero}

Under Assumption \ref{smp:main}, we model the wake of each UAV $i\in\mathcal{A}$, as a scalar field centered at $\bm{p}_i$ and aligned with $\theta_i$.
Physically, the upwash field is a consequence of the pressure difference between the top and bottom of the wing \cite{Anderson2017FundamentalsAerodynamics}. 
This induces a vortex at the wing tips, which generates an upward velocity (i.e., upwash) far from the wing and downward velocity (i.e., downwash) behind the wing.
Classically the wingtip vortices have been modeled using irrotational flow \cite{Mirzaeinia2019EnergyReconfiguration,Anderson2017FundamentalsAerodynamics,Karamcheti1980PrinciplesAerodynamics}.
However, this model is known to cause nonphysical behavior.
In particular, the vertical air speed approaches infinity at the wing tips.
In this work we augment the irrotational vortex model with a rotational core, which drives the velocity to zero at the wing tips.
Namely, each vortex induces the upwash velocity,
\begin{equation} \label{eq:spanAero}
    u_i(r) = 
    \begin{cases}
    \frac{\Gamma}{2\pi r} & \text{ if } |r| \geq r^*, \\
    \Omega r & \text{ if } |r| \leq r^*,
    \end{cases}
\end{equation}
where $u_i$ is the vertical airspeed, $r$ is the distance to the vortex center, $\Gamma$ is the circulation of the irrotational vortex, $\Omega$ is the angular rotation speed of the rotational core, and
$r^*$ satisfies
\begin{equation}
    r^* = \Big(\frac{\Gamma}{2\pi\Omega}\Big)^{\frac{1}{2}}.
\end{equation}
Note that $r^* << b$ for physical systems, and under Assumption \ref{smp:main}, the induced velocity field has the form,
\begin{equation} \label{eq:spanwise}
    f(y) = u_i(y-b) - u_i(y+b),
\end{equation}
where $y$ is a relative spanwise position and $2b$ is the wingspan.
The upwash distribution is shown in Fig. \ref{fig:upwash}.

\begin{figure}[ht]
    \centering
    \includegraphics[width=\linewidth]{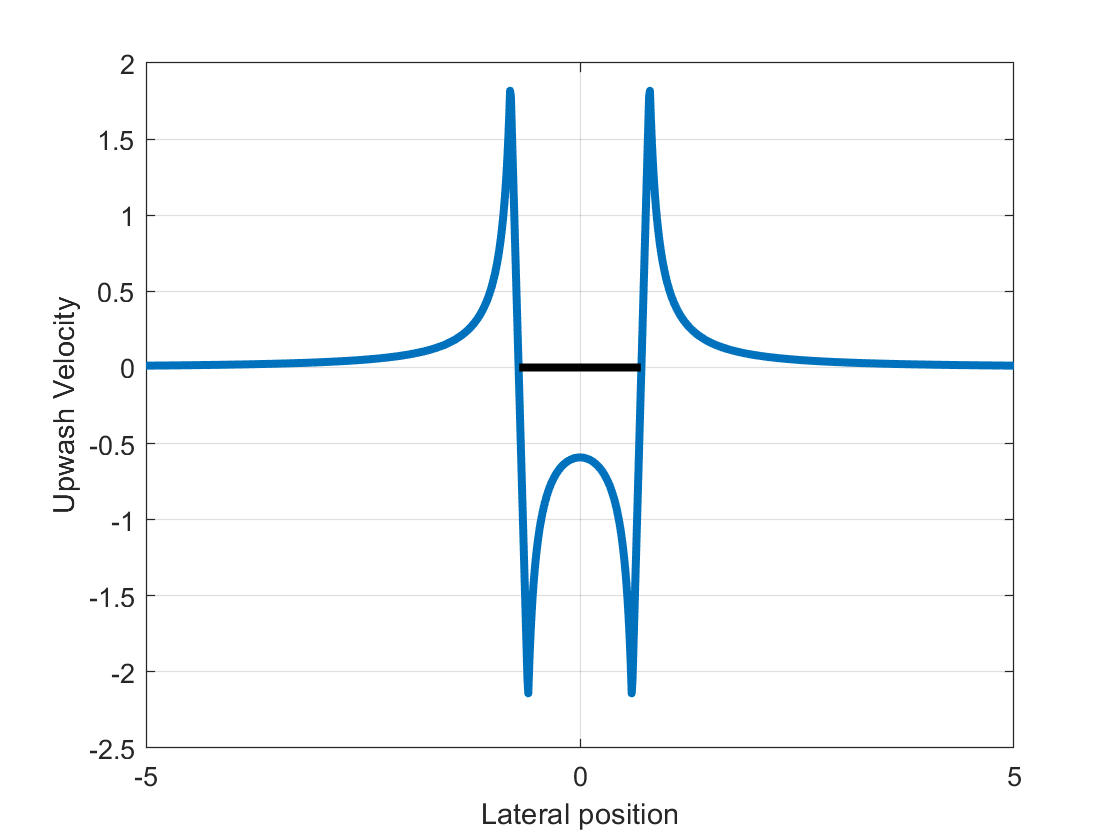}
    \caption{Upwash velocity induced in the spanwise direction due to the wing tip vortices.}
    \label{fig:upwash}
\end{figure}

In the streamwise, i.e., longitudinal direction, the wing interactions quickly coalesce into the two strong wingtip vortices.
As the vortices move aft from the wing, their influence approximately doubles.
Modeling the streamwise behavior of the vortices is a challenging problem; we adopt the approach of \cite{Yang2016LoveControl} using a gaussian function,
\begin{equation} \label{eq:streamwise}
    g(x) = 2 \exp{-\frac{(x-\mu)^2}{2\sigma^2}},
\end{equation}
where $\mu$ determines the location of maximum upwash benefit, $\sigma$ determines the length of the wake, and $x$ is the relative streamwise position,.
Finally, we define the relative position vector,
\begin{equation}
    \bm{s}_{ij}(\bm{p}_i, t) = \bm{p}_i - \bm{p}_j(t),
\end{equation}
and combine \eqref{eq:spanwise} and \eqref{eq:streamwise} to find an expression for the magnitude of UAV $j$'s upwash on $i$,
\begin{equation}
    w_i^j(\bm{p}_i, t) \coloneqq f\Big(\bm{s}_{ij}(\bm{p}_i, t)\cdot\hat{y}\Big)\cdot g\Big(\bm{s}_{ij}(\bm{p}_i, t)\cdot\hat{x}\Big),
\end{equation}
where $\hat{x}, \hat{y}$ are unit vectors aligned with and perpendicular to $\theta_{j}(t)$, respectively.

\section{Optimal Feedback Controller} \label{sec:control}

To simplify our analysis of the system-level behavior, we impose the following pairwise assumption on the UAV positions.
\begin{assumption} \label{smp:pairwise}
    For each UAV $i\in\mathcal{A}$, there is at most one UAV $j\neq i$ such that the upwash force $w_i^j(\bm{p}_i, t)$ is not negligible.
\end{assumption}
Intuitively, Assumption \ref{smp:pairwise} requires the UAVs to be sufficiently `close' to a V or echelon formation.
This makes our analysis tractable, as each UAV must only consider the influence of its immediate leader.
We extend the results of this section to the general case in Section \ref{sec:experiment}, where the UAVs are simulated in 2D space and the agents' initial conditions do not satisfy Assumption \ref{smp:pairwise}.
Finally, for completeness, we define the necessary conditions for a flock of UAVs to remain stable under Assumption \ref{smp:main}.

\begin{definition} \label{def:stable}
A necessary condition for the flock of UAVs is \emph{stable} at time $t_0$ under Assumption \ref{smp:main} is, for each UAV $i\in\mathcal{A}$,
\begin{equation}
    v_i(t) = v_i(t_0), \quad \theta_i(t) = \theta_j(t)
\end{equation}
for all $t \geq t_0$.
\end{definition}

\subsection{Drag Minimization}

As a first step, we seek the control input $v_i$ that minimizes the drag force on UAV $i\in\mathcal{A}$, i.e.,
\begin{equation} \label{eq:fDot}
    \frac{\partial F_i}{\partial v_i} = 2C_1 v_i - 2 \frac{C_2}{v_i^3} + \frac{L}{v_i^2} W_i(\bm{p}_i, t) = 0.
\end{equation}
Rearranging terms and multiplying by $v_i^3$ yields,
\begin{equation} \label{eq:vMinDrag}
    v_i^4 + \frac{L}{2 C_1} W_i(\bm{p}_i,t) v_i - \frac{C_2}{C_1} = 0.
\end{equation}

\begin{remark}
    Note that \eqref{eq:vMinDrag} minimizes the drag experienced by UAV $i\in\mathcal{A}$, which generally maximizes the distance traveled by the UAV per unit of energy expended.
    Alternatively, one could minimize the power lost to drag by considering the product of \eqref{eq:drag} and $v_i$. This generally maximizes the flight time of the UAV.
    The following analysis holds for both cases.
\end{remark}

\begin{lemma} \label{lma:existance}
    There is a unique real positive speed that minimizes the drag experienced by each UAV $i\in\mathcal{A}$.
\end{lemma}

\begin{proof} 

The optimal airspeed for UAV $i$ is the solution to \eqref{eq:vMinDrag}, a quartic function of $v_i$ with the discriminant 
\begin{equation}
    \Delta_4 = - \Bigg(256 \frac{C_2}{C_1} + 27\Big(-\frac{L}{2C_1}W_i(\bm{p}_i, t)\Big)^4 \Bigg) > 0,
\end{equation}
which implies that \eqref{eq:vMinDrag} has two complex conjugate roots and two real roots.
The imaginary roots satisfy,
\begin{equation} \label{eq:imaginaryRoots}
    v_i^2 + b v_i + c = 0, \quad b^2 - 4c < 0,
\end{equation}
and thus $c > 0$ .
Next, polynomial long division of \eqref{eq:vMinDrag} on \eqref{eq:imaginaryRoots} yields a quadratic form for the real roots and additional conditions on $a$ and $b$, i.e.,
\begin{align}
    v_i^2 - bv + b^2 - c &= 0, \label{eq:realRoots} \\
    2bc + \frac{L}{2C_1}W_i(\bm{p}_i, t) - b^3 &= 0, \\
    c^2 - b^2 c - \frac{C_2}{C_1} &= 0. \label{eq:lma1C}
\end{align}
Condition \eqref{eq:lma1C} implies,
\begin{equation}
    c(c - b^2) = \frac{C_1}{C_2} > 0,
\end{equation}
which, in turn, implies $c > b^2$.
Finally, applying the quadratic equation to \eqref{eq:realRoots} yields,
\begin{equation}
    v_i = \frac{b \pm \sqrt{-3b^2 + 4c}}{2}.
\end{equation}
Multiplying the two real roots yields,
\begin{equation}
    \frac{1}{4}\Big(4b^2 - 4c \Big) = b^2 - c < 0.
\end{equation}
Thus, the two real solutions to $v_i$ have opposite signs, and \eqref{eq:vMinDrag} has exactly one real positive solution.
\end{proof}

Note that following the proof of Lemma \ref{lma:existance} it is possible to derive the optimal airspeed analytically; however, this is beyond the scope of this article.
Our next result characterizes how the upwash benefit affects the optimal airspeed of each UAV.

\begin{lemma} \label{lma:upwashSpeed}
The optimal airspeed of UAV $i$ decreases when gaining an upwash benefit and increases when experiencing an upwash cost.
\end{lemma}

\begin{proof}
Consider a UAV $i\in\mathcal{A}$ flying in isolation.
In this case $W(\bm{p}_i, t) = 0$, and the optimal airspeed arises when \eqref{eq:fDot} is satisfied, i.e.,
\begin{equation}
    v_i^* = \Big(\frac{C_2}{C_1}\Big)^{1/4}.    
\end{equation}
Substituting this into \eqref{eq:fDot} and rearranging terms implies,
\begin{equation}
    v_i^4 - (v_i^*)^4 = -\frac{L}{2 C_1} v_i W_i(\bm{p}_i, t).
\end{equation}
Thus, as $v_i > 0$ from Lemma \ref{lma:existance},  $W_i < 0$ (upwash cost) implies $v_i > v_i^*$ and $W_i > 0$ (upwash benefit) implies $v_i < v_i^*$.
\end{proof}

\begin{lemma} \label{lma:unstableUpwash}
    Under Assumption \ref{smp:main}, flying at the optimal airspeed while receiving an upwash benefit is unstable (Definition \ref{def:stable}), and each UAV will end up with no upwash benefit after finite time.
\end{lemma}

\begin{proof}
For a UAV formation to be stable, it is necessary that $v_i = v_j,$ $\theta_i = \theta_j$ for all UAVs, $i,j\in\mathcal{A}$.
Under our premise, this implies $v_i^* = v_j^*$ for all $i,j\in\mathcal{A}$ and,
\begin{align*}
    (v_i^*)^4 + \frac{L}{2 C_1} W_i(\bm{p}_i, t) v_i^* &- \frac{C_2}{C_1} = 
    0 \\
    &= (v_j^*)^4 + \frac{L}{2 C_1} W_j(\bm{p}_j, t) v_j^* - \frac{C_2}{C_1},
\end{align*}
which implies,
\begin{equation} \label{eq:nec1}
    W_i(\bm{p}_i, t) = W_j(\bm{p}_j, t).
\end{equation}
The latter is an equivalent necessary condition for stability.
Without loss of generality, let $\theta_i = \theta_j = 0$ for all $i, j \in \mathcal{A}$.
Consider the case when UAV $i$ is ahead of another UAV $j$, i.e.,
\begin{equation} \label{eq:ahead}
    (\bm{p}_i - \bm{p}_j) \cdot
    \begin{bmatrix}
    1 \\ 0
    \end{bmatrix}
    > 0,
\end{equation}
where $W_i = W_j \neq 0$, and the upwash effect of $i$ upon $j$ is not negligible.
In this case, the streamwise asymmetry of the upwash field requires the existence of another UAV $k$ such that $k$ as a negligible impact on $j$ (Assumption \ref{smp:pairwise}) and a significant impact on $i$; this can only be achieved if UAV $k$ is ahead of $i$ in the streamwise direction.
Repeating this process implies an uncountably infinite number of UAVs, which contradicts our premise that $N$ is finite.
\end{proof}

\begin{theorem} \label{thm:noV}
Under Assumption \ref{smp:main}, flying at the optimal air speed never leads to an emergent V formation.
\end{theorem}

\begin{proof}

Let every UAV $i\in\mathcal{A}$ have a state such that the position $\bm{p}_i$ corresponds to a V formation where no two agents overlap, each UAV receives some upwash benefit, and Assumption \ref{smp:pairwise} is satisfied.
Let $j$ denote the frontmost agent and $i$ satisfy $w_i^j(\bm{p}_i, t) > 0$.
Then, by the asymmetry of the streamwise upwash benefit \eqref{eq:streamwise}, $W_i > W_j$.
Thus, by Lemma \ref{lma:upwashSpeed}, $v_i^* < v_j^*$, and the distance $||\bm{s}_{ij}(\bm{p}_i, t)||$ is increasing.
This implies that the UAVs will travel at different optimal air speeds until the V formation breaks apart.
\end{proof}

Theorem \ref{thm:noV} demonstrates that simply flying at the energy-optimal airspeed can never lead to emergent line flocking!
This highlights a significant shortcoming within applying the robot ecology approach \cite{Egerstedt2021RobotAutonomy} to line flocking.
Specifically, if each UAV $i\in\mathcal{A}$ minimizes its acceleration subject to a constraint that $v_i$ matches $v_i^*$ as closely as possible, then an energy-saving V formation can not occur.
A recent game-theoretic approaches has run into similar challenges \cite{Shi2021AreFormation}.
Therefore, rather than minimizing the ``energy'' spent to actuate by minimizing acceleration, we propose that each agent ought to minimize its ``locomotive power'' expended through the cost function,
\begin{align} \label{eq:costFunction}
    J(\bm{x}_i) &= \Big(\frac{v_i - v_i^*}{v_{\max} - v_{\min}}\Big)^2 + \Big(\frac{\omega_i}{\omega_{\max}}\Big)^2,
\end{align}
where $v_i^*$ is the unique optimal airspeed (Lemma \ref{lma:existance}) and both cost components are normalized because they are orthogonal control vectors.
Employing the cost \eqref{eq:costFunction} is a subtle change from the state of the art \cite{Egerstedt2021RobotAutonomy}, but it ends up playing a critical role in the generation and stabilization of emergent V formations.

\subsection{Wake Interactions}

Next, we explore the wake interaction of UAV $j$ on $i \in\mathcal{A}$ using our simplified aerodynamic model.
To simplify our notation, we use the scalars $x_i$ and $y_i$ to denote the relative position of $i$ with respect to $j$ in the streamwise and spanwise directions, respectively.
First, we estimate how UAV $i$ tends to roll due to the local flow field by evaluating the integral,
\begin{equation} \label{eq:moment}
    m_i(x_i, y_i, t) = g(x_i) \int_{y_i-b}^{y_i+b} (\xi-y_i) f(\xi, t) d\xi.
\end{equation}
Similarly, we can estimate the lift induced on the wing through momentum transfer,
\begin{equation} \label{eq:force}
    w_i(x_i, y_i, t) = g(x_i) \int_{y_i-b}^{y_i+b} f(\xi, t)  d\xi.
\end{equation}
Note that both \eqref{eq:moment} and \eqref{eq:force} have analytical closed-form solutions, as the irrotational flow far from the wing-tip is integrable and it transitions to an affine function near the wing tip.
The resulting upwash and tendancy to roll functions are visualized in Fig. \ref{fig:moment}.
Taking the sum of \eqref{eq:force} over all UAVs determines the aggregate upwash effect on $i$,
\begin{equation} \label{eq:fields}
\begin{aligned}
    W_i(\bm{p}_i, t) = \sum_{k\in\mathcal{A}\setminus\{i\}} w_k(\bm{p}_i, t), \\
    M_i(\bm{p}_i, t) = \sum_{k\in\mathcal{A}\setminus\{i\}} m_k(\bm{p}_i, t),
\end{aligned}
\end{equation}
where $\bm{p}_i$ must be projected onto the streamwise and spanwise components of UAV $k$ to yield $x_i$ and $y_i$.
Note that \eqref{eq:fields} should be multiplied by a constant to compute the upwash force in the correct units.
However, our results only rely on the sign of the derivative of each term, so any positive scaling factor is irrelevant to our analysis.
Furthermore, the sums of \eqref{eq:fields} contain at most one term under Assumption \ref{smp:pairwise}.

\begin{figure}[ht]
    \centering
    \includegraphics[width=\linewidth]{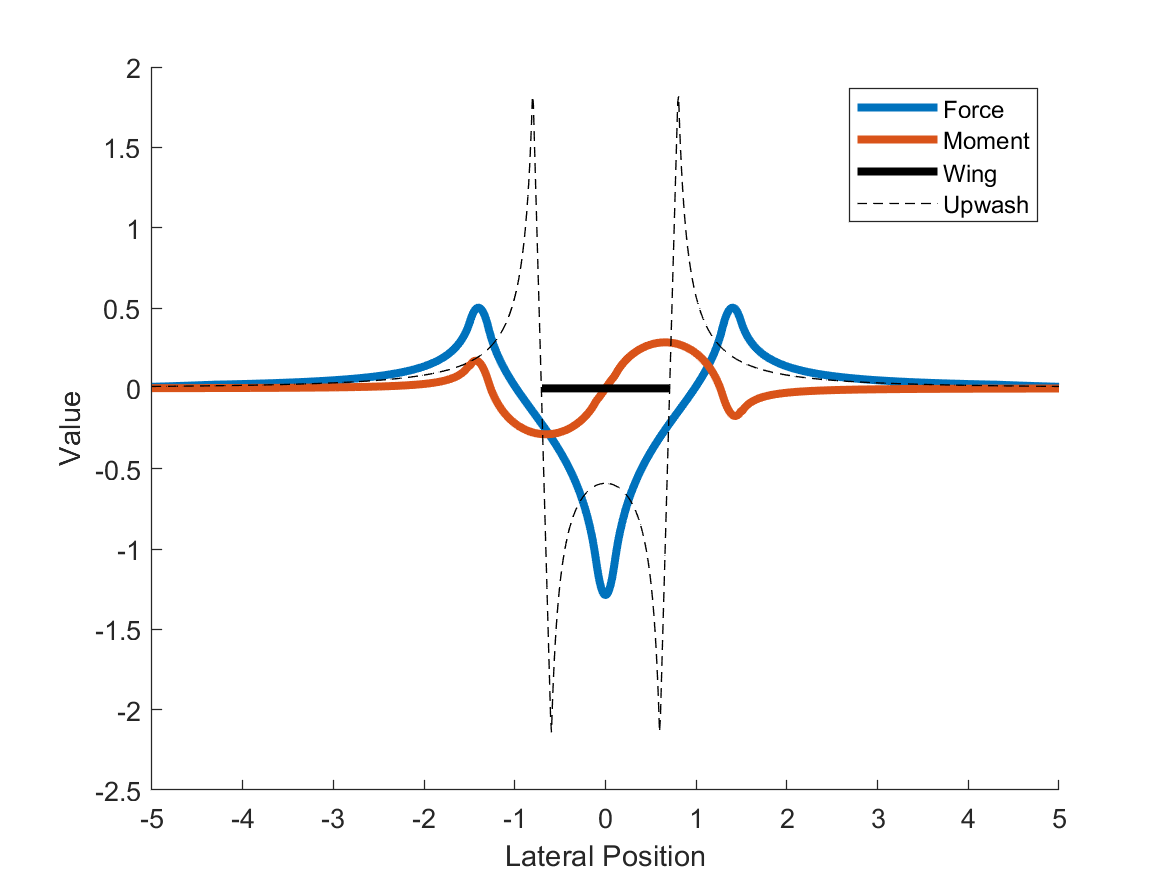}
    \caption{Upwash force and tendancy to roll calculated by integrating the upwash velocity field along the wingspan at each point in the domain.}
    \label{fig:moment}
\end{figure}

Next, we estimate the cost required for UAV $i$ to maintain a constant altitude and orientation in the presence of the upwash field.
In particular, the UAV must expend energy to counter-roll against gradients in the upwash field, and it must pitch upward to counteract a negative upwash (downwash).
This leads to an intuitive physics-based description of the ``cost to flock,''
\begin{equation}
    E_i(p_i, t) = \kappa |M_i(p_i, t)| - W_i(p_i, t),
\end{equation}
where $\kappa$ is a system parameter that captures the tradeoff between the cost to roll and the cost to pitch upward.
We require each UAV to satisfy,
\begin{equation}
    \dot{E}_i(\bm{p}_i, t) \leq 0,
\end{equation}
where $E_i(\bm{p}_i, t)$ has a finite lower bound.
Thus, each UAV is driven toward an equilibrium point where the energy lost through wake interactions is minimized, or equivalently, the energy saved by flocking is maximized.

\begin{lemma} \label{lma:zeroUpwash}
For two agents $i, j$ such that $\bm{p}_j = \bm{0}$, there exists a unique lateral position $\hat{y}_i > 0$ such that:
\begin{align}
    y_i < \hat{y} \implies w_i(\cdot, y_i, \cdot) < 0, \\
    y_i = \hat{y} \implies w_i(\cdot, y_i, \cdot) = 0, \\
    y_i > \hat{y} \implies w_i(\cdot, y_i, \cdot) > 0.
\end{align}
\end{lemma}

\begin{proof}
By definition $f_i(y_i)$ satisfies,
\begin{equation}
    \begin{aligned} \label{eq:yiBehavior}
        0\leq y_i < b \implies f_i(y_i) < 0, \\
        y_i = b \implies f_i(y_i) = 0, \\
        y_i > b \implies f_i(y_i) > 0.
    \end{aligned}
\end{equation}
Thus, the equation,
\begin{equation} \label{eq:upwashLemma}
    \int_{y_i-b}^{y_i+b} f(\xi) d\xi = 0,
\end{equation}
has a solution on the interval $y_i \in (0, 2b)$.
Let $\hat{y}_i$ denote the smallest solution to \eqref{eq:upwashLemma}, then, for any $\epsilon > 0$, the upwash difference is,
\begin{align*}
    D(\epsilon) =& \int_{\hat{y}_i-b}^{\hat{y}_i+b} f(\xi) d\xi - \int_{(\hat{y}_i+\epsilon)-b}^{(\hat{y}_i+\epsilon)+b} f(\xi) d\xi \\
    &= \int_{\hat{y}_i-b}^{\hat{y}_i-b+\epsilon} f(\xi) d\xi - \int_{\hat{y}_i+b}^{\hat{y}_i+b+\epsilon} f(\xi) d\xi.
\end{align*}
Note that by \eqref{eq:yiBehavior} the first integral is negative, as $\hat{y}_i \in (0, 2b).$
Meanwhile the second integral is positive, and thus $D(\epsilon)$ is negative for $\epsilon > 0$.
This implies that $\hat{y}_i$ is the unique solution, and $y_i > \hat{y_i} \implies w_i(\cdot, y_i, \cdot) > 0$.
Repeating the process for $\epsilon < 0$ completes the proof.
\end{proof}

\begin{lemma} \label{lma:zeroMoment}
For two agents $i, j$ such that $\bm{p}_j = \bm{0}$, there exists a unique lateral position $\hat{y}_i > b$ such that:
\begin{align}
    y_i > \hat{y} \implies m_i(\cdot, y_i, \cdot) < 0, \\
    y_i = \hat{y} \implies m_i(\cdot, y_i, \cdot) = 0, \\
    y_i < \hat{y} \implies m_i(\cdot, y_i, \cdot) > 0.
\end{align}
\end{lemma}

\begin{proof}
The proof of Lemma \ref{lma:zeroMoment} is similar to that of Lemma \ref{lma:zeroUpwash}, and we omit it for brevity.
\end{proof}

\begin{theorem} \label{thm:equilibrium}
    Consider two UAVs $i, j \in \mathcal{A}$ such that $\bm{p}_j = \bm{0}$ and the upwash effect of $j$ on $i$ is not negligible.
    Let $[x_i, y_i]^T = \bm{p}_i - \bm{p}_j$ denote the streamwise and lateral position of $i$ relative to $j$. Furthermore, let the constant $\kappa$ be sufficiently small to admit a position $x_i^*, y_i^*$ such that $\nabla E_i(x_i^*, y_i^*, \cdot) = \bm{0}$ and there is a benefit to flocking, i.e., $E_i(x_i^*, y_i^*, \cdot) < 0$.
    Then, any equilibrium point that yields $\nabla E_i(x_i^*, y_i^*, \cdot) = 0$ satisfies $x_i^* = -\mu$ $y_i^* > \sqrt{2}\,b$, which corresponds to a V or echelon formation.
\end{theorem}

\begin{proof}
    By definition $g(x_i,t)$ is log-convex and has a unique maximum at $x_i = -\mu$.
    Next, let $b + r^* \leq y_i \leq 2b - r^*$; the upwash experienced by the UAV is,
    \begin{align}
        w_i = g(\cdot)\Big(
        \int_{y_i-b}^{b-r^*} u_i(\xi-b) d\xi
        + \int_{b-r^*}^{b+r^*}  u_i(\xi-b) d\xi \notag \\
        + \int_{b+r^*}^{y_i+b}  u_i(\xi-b) dz\xi
        - \int_{}^{}  u_i(\xi+b) d\xi
        \Big),
    \end{align}
    where the second integral is zero by symmetry of the integrand about $(\xi-b) = 0$.
    This implies that
    \begin{align}
        w_i = g(\cdot)\frac{\Gamma}{2\pi}\Big(
        \ln(\frac{r^*}{2b - y_i}) + \ln(\frac{y_i}{r^*}) - \ln(\frac{y_i+2b}{y_i})
        \Big) \notag\\
        = g(\cdot)\frac{\Gamma}{2\pi}\ln\Big( \frac{y_i^2}{4b^2 - y_i^2} \Big).
    \end{align}
    Taking a derivative with respect to $y_i$ yields,
    \begin{equation}
        \frac{\partial w_i}{y_i} = g(\cdot)\frac{\Gamma}{2\pi}\Big( \frac{8b^2}{y(4b^2 - y^2)} \Big),
    \end{equation}
    which is positive for $y_i < 2b$, thus $w_i$ is strictly increasing for $y_i\in(b+r^*, 2b-r^*)$.
    Furthermore, Lemma \ref{lma:zeroUpwash} implies that $w_i = 0$ has a unique solution, which occurs at $y_i = \sqrt{2}\,b$. This constitutes the lower bound on $y_i$ by our premise.
\end{proof}

Finally, based on our analysis of the upwash field properties, we formally define the rules of anseroid behavior to generate emergent V and echelon formations.
\begin{definition} \label{def:anseroid}
    Anseroid behavior is characterized by two rules: (1) maximize the energy benefit of the local upwash field, and (2) match the drag-minimizing speed as closely as possible.
\end{definition}

\subsection{Optimal Control Problem}

We employ \emph{constraint-driven control} to implement anseroid behavior for each UAV.
Constraint-driven control is an optimization-based control technique, wherein the rules of interaction are embedded as constraints in an optimization problem.
This technique has been used successfully to control multi-agent constraint-driven systems, particularly in the ecologically-inspired robotics literature \cite{Notomista2019Constraint-DrivenSystems,Wang2017SafetySystems,Beaver2021Constraint-DrivenStudy}.
Our motivation for constraint-driven control is twofold:
first, it enables the UAVs to immediately react to their surroundings without the computational and communication costs associated with decentralized trajectory planning \cite{Beaver2020AnFlockingb,Beaver2022Constraint-DrivenAvoidance}.
This has the added benefit of allowing UAVs to be arbitrarily added to and removed from the domain without a priori knowledge \cite{Beaver2021Constraint-DrivenStudy}.
Second, implementing our desired behavior with constraints can allow for strong guarantees on the system-level behavior, unlike traditional optimization approaches where the objective is composed of performance criteria and soft constraints.
Thus, to implement the anseroid rules of Definition \ref{def:anseroid}, we propose Problem \ref{prb:optCtrl}.

\begin{problem} \label{prb:optCtrl}
Each UAV $i\in\mathcal{A}$ takes the optimal control input that optimizes,
\begin{align*}
\min_{v_i, \omega_i} & \Bigg\{ \Big(\frac{v_i - v_i^*}{v_{\max} - v_{\min}}\Big)^2 + \Big(\frac{\omega_i}{\omega_{\max}}\Big)^2 \Bigg\} \\
\text{subject to: } &
\eqref{eq:dynamics}, \eqref{eq:bounds}, \\
&\dot{E}(\bm{p}_i, t) \leq \rho \\
&|\theta_i - {\theta}_g| \leq \epsilon,
\end{align*}
where $\theta_g$ is the global heading angle (Assumption \ref{smp:main}) and $\rho \leq 0$ determines how quickly the UAVs approach the equilibrium point $\dot{E} = 0$.
Both terms are selected by a designer a priori.
\end{problem}

Note that while our analysis of the system behavior is in continuous-time, in many cases the optimal control problem must be formulated using discrete-time system dynamics.
This mapping is a practical consideration for implementation on a physical robot, and it can be a theoretical concern when the state variables, rather than the control input, appear in the constraints and objective function \cite{Beaver2022Constraint-DrivenAvoidance}.
This is an area of open research in the literature. Potential solutions include allowing the time step to grow arbitrarily small \cite{Xu2022FeasibilityFunctions}, employing control barrier functions to convert state constraints into control constraints \cite{Notomista2022Multi-robotTasks}, or tightening the constraints at each time step to ensure the trajectory remains feasible between time steps \cite{Pant2017SmoothLogic}.

In our next result, we provide a necessary and sufficient condition for the feasible space of Problem \ref{prb:optCtrl} to be empty, and we use this as the foundation for a switching control law.

\begin{theorem} \label{thm:feasible}
For UAV $i\in\mathcal{A}$ satisfying $|\theta_i - \theta_g|\leq \epsilon$, the feasible space of Problem \ref{prb:optCtrl} is empty if and only if the inequality,
\begin{equation} \label{eq:existanceThmStatement}
   v_i \, L_{\bm{f}_i}E_i \leq \rho - \frac{\partial E_i}{\partial t},
\end{equation}
does not hold for any
\begin{equation}
    v_{\min} \leq v_i \leq v_{\max},
\end{equation}
where $\bm{f} = [\cos(\theta_i), \sin(\theta_i)]^T$ denotes the forward direction of UAV $i$ and $L_{\bm{f}} E$ is the Lie derivative of $E_i$ in the direction of $\bm{f}$.
\end{theorem}

\begin{proof}
Any solution to Problem \ref{prb:optCtrl} must satisfy the inequality constraint $\dot{E}_i(\bm{p}_i, t) \leq \rho$, which expands to,
\begin{equation}
    \dot{E}_i = \frac{\partial E_i}{\partial t} + \frac{\partial E_i}{\partial \bm{p}_i} v_i
    \begin{bmatrix}
    \cos{\theta_i} \\ \sin{\theta_i}
    \end{bmatrix} \leq \rho,
\end{equation}
by the definition of the full derivative.
Rearranging terms implies,
\begin{equation} \label{eq:existanceThmProof}
    v_i
    \frac{\partial E_i}{\partial \bm{p}_i}\begin{bmatrix}
    \cos{\theta_i} \\ \sin{\theta_i}
    \end{bmatrix}
    \leq  \rho - \frac{\partial E_i}{\partial t},
\end{equation}
which is equal to \eqref{eq:existanceThmStatement}.
Next, any solution to Problem \ref{prb:optCtrl} must also satisfy the bounds,
\begin{equation} \label{eq:thetaBoundProof}
    \quad v_{\min} \leq v_i \leq v_{\max}.
\end{equation}
Thus, if \eqref{eq:existanceThmProof} and \eqref{eq:thetaBoundProof} cannot be jointly satisfied, then the feasible space of Problem \ref{prb:optCtrl} is empty.
Similarly, if any control action $v_i^*$ jointly satisfies \eqref{eq:existanceThmProof} and \eqref{eq:thetaBoundProof}, then $v_i^*$ is a feasible solution of Problem \ref{prb:optCtrl}.
\end{proof}

Intuitively, Theorem \ref{thm:feasible} does not allow UAV $i$ to move in a direction that locally decreases the scalar cost field $E_i$ by less than $|\rho + \frac{\partial E_i}{\partial t}|$.
At each time instant the values of $L_{\bm{f}}E_i$ and $\frac{\partial E_i}{\partial t}$ are fixed by the state of UAV $i$ and the environment, and the UAV must select a value of $v_i$ that $E_i$ is locally decreasing.
In this context, one can interpret $v_i$ as a step size that must overcome the disturbances introduced through $\frac{\partial E_i}{\partial t}$.

Theorem \ref{thm:feasible} describes the conditions where the cost-minimizing constraint is incompatible with the velocity bounds of the agent.
This is closely related to the constraint compatibility problem in the control barrier functions literature, and it is well-studied in the set-theoretic control community \cite{Wang2017SafetySystems,Xiao2022SufficientFunctions,Ibuki2020Optimization-BasedBodies,Beaver2022Constraint-DrivenAvoidance}.
Generally, the problem of constraint incompatibility has been solved in the ecologically-inspired robotics literature by introducing slack variables \cite{Egerstedt2018RobotAutonomy,Ibuki2020Optimization-BasedBodies,Egerstedt2021RobotAutonomy}.
However, this is not fundamentally different from moving the constraint into the objective function to make it ``soft."
A foundational paper in multi-agent control barrier functions proposed operating the system in two modes \cite{Wang2017SafetySystems}: 1) a nominal mode where the agents solve the optimal control problem, and 2) a ``safe mode" where the agents come to a stop when the feasible space is empty.
We take this approach to its logical conclusion -- when the feasible space of Problem \ref{prb:optCtrl} is empty, the controller switches modes and solves a relaxed version of the problem.
This implies an equivalent switched system \cite{Beaver2022Constraint-DrivenAvoidance}, where the UAV uses the premise of Theorem \ref{thm:feasible} to switch between the full and relaxed problem, which is illustrated in Fig. \ref{fig:stateMachine}.

\begin{figure}[ht]
    \centering
    \includegraphics[width=0.7\linewidth]{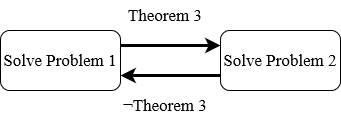}
    \caption{The behavior of each UAV visualized as a switched system. Theorem \ref{thm:feasible} determines when the UAV should solve the original or the relaxed optimal control problem.}
    \label{fig:stateMachine}
\end{figure}

When Problem \ref{prb:optCtrl} has no feasible solution, it is unreasonable to relax the constraints corresponding to the dynamics or control bounds. The only reasonable options are to relax $\dot{E}_i \leq \rho$ or $|\theta_i - \theta_g| \leq \epsilon$.
We propose that the former should be relaxed to maintain Assumption \ref{smp:main}.
Next, we present the relaxed optimal control problem with Problem \ref{prb:relaxed}, followed by a result that shows such a relaxation only lasts for a finite time interval.

\begin{problem} \label{prb:relaxed}
Each UAV $i\in\mathcal{A}$ takes the optimal control input that optimizes,
\begin{align*}
\min_{v_i, \omega_i} & \Bigg\{ \Big(\frac{v_i - v_i^*}{v_{\max} - v_{\min}}\Big)^2 + \Big(\frac{\omega_i}{\omega_{\max}}\Big)^2 \Bigg\} \\
\text{subject to: } &
\eqref{eq:dynamics}, \eqref{eq:bounds}, \\
&|\theta_i - {\theta}_g| \leq \epsilon,
\end{align*}
where $\theta_g$ is the global heading angle (Assumption \ref{smp:main}).
\end{problem}

If the state UAV $i$ satisfies $|\theta_i - \theta_g| \leq \epsilon$, it is trivial to show that the control action
\begin{equation}
\begin{aligned}
    v_i &= \max\big\{\min\big\{ v^*, v_{\max}   \big\}, v_{\min} \big\},\\
    \omega_i &= 0,
\end{aligned}
\end{equation}
is the optimal solution to Problem \ref{prb:relaxed}.
Note that the solution to Problem \ref{prb:relaxed} corresponds to the unstable case presented in Lemma \ref{lma:unstableUpwash}, where the UAV greedily matches the optimal upwash speed $v^*$.
This mode of operation corresponds to a ``safe mode" that retains some structure of the original problem, i.e., it doesn't require the UAVs to come to a complete stop, without guaranteeing formation flight.
Next, we prove that each UAV will only need to solve Problem \ref{prb:relaxed} for a finite amount of time before the premise of Theorem \ref{thm:equilibrium} is no longer satisfied.

\begin{corollary} \label{cor:finiteTime}
Any ``following'' UAV will only solve Problem \ref{prb:relaxed} for a finite interval of time before the premise of Theorem \ref{thm:feasible} is not satisfied.
\end{corollary}

\begin{proof}
Corollary \ref{cor:finiteTime} follows trivially from the proof Lemmas \ref{lma:upwashSpeed} and \ref{lma:unstableUpwash}; any following UAV $i$ will have a higher upwash benefit, and thus lower optimal airspeed, than the lead UAV.
This implies that the following UAV satisfies $W_i = 0$ after a finite amount of time, which implies that $\dot{E} = 0$ for all control actions.
\end{proof}

Thus, we have introduced Theorem \ref{thm:feasible}, which gives necessary and sufficient conditions for the system to switch between the flocking controller (Problem \ref{prb:optCtrl}) and the safe controller (Problem \ref{prb:relaxed}).
Furthermore, Corollary \ref{cor:finiteTime} demonstrates that the UAVs will spend a finite amount of time solving the relaxed control problem.



\subsection{On Heterogeneity} \label{sec:hetero}

One compelling result of Lemma \ref{lma:upwashSpeed} is the fact that introducing heterogeneity has the potential to further improve the fuel consumption of all UAVs in the flock.
It demonstrates that any UAV gaining an upwash benefit has a lower optimal speed than when flying in isolation, i.e., each $v_i^*$ decreases as $W_i$ increases for each UAV $i\in\mathcal{A}$.
However, each UAV must match the speed of the UAV in front of it to maintain the formation.
This implies that the flock will inevitably fly at the speed of the frontmost agent, with a speed that is significantly higher than $v_i^*$ of the following agent $i$.
However, if heterogeneity is introduced to the flock--through aerodynamic design, payload weight, or fuel consumption--then the slowest agents can be placed at the front of the flock.
This enables the faster agents in the rear, who receive significantly larger upwash benefits, to fly slower and have a steady-state speed closer to their optimal airspeed.
In this way, heterogineity in the optimal airspeed of the UAVs can have benefits across the entire flock.

\section{Experimental Validation} \label{sec:validation}

Next, we present two simulations and a physical experiment using our proposed flocking controllers.
The first simulation is for $N=2$ UAVs in Matlab, and it demonstrates that pairs of agents converge to the equilibrium points described by Theorem \ref{thm:equilibrium}.
Next, we demonstrate a Python-based simulation that uses the CrazySwarm Python API \cite{crazyswarm} for $N=11$ UAVs; this demonstrates that our proposed flocking controller is platform independent, and that the controller scales to larger numbers of agents.
Finally, we run the Python controller in real-time on a fleet of $N=5$ CrazyFlie quadrotors.
The quadrotors are a surrogate for fixed-wing UAVs, and with them we demonstrate that the control algorithm easily runs in real time--including the time it takes to simulate the aerodynamic interactions between agents.

\subsection{Implementation Issues} \label{sec:implementation}

We foresee two major issues when solving Problem \ref{prb:optCtrl} on a real system.
The first is a sensing issue, namely determining the value of $M_i$ and $W_i$ for each UAV $i\in\mathcal{A}$.
We simulate $M_i$ and $W_i$ explicitly in our results, but for a real fixed-wing UAV we proposed a solution that relies on Assumption \ref{smp:tracking}, i.e., that the UAV is equipped with a low-level tracking controller.
At the tracking level, the upwash force and induced roll act as disturbances on the UAV.
By monitoring the roll and pitch signals, we conjecture that the corresponding upwash force and induced tendency to roll can be inferred from the flight controller.
In other words, our proposed anseroid controller is an outer loop that samples the low-level flight controller data.

The second issue is information-theoretic, namely calculating the time derivative of $E_i$ for each UAV $i\in\mathcal{A}$.
The gradient of the upwash field could possibly be predicted using an aerodynamic model, however this does not capture the time-varying component,
\begin{equation}
    \dot{E} = \frac{\partial E}{\partial t} + \frac{\partial E}{\partial \bm{p}_i} \dot{\bm{p}}_i,
\end{equation}
where the non-smooth points in $\frac{d E}{d \mathbf{p}}$ can be handled with differential inclusions \cite{Santos2019DecentralizedFunctions}.
The main problem is that UAV $i$ must also have knowledge of $v_j$ to compute $\frac{\partial E_i}{\partial t}$.
Similarly, UAV $j$ must have knowledge of $v_i$ to compute $\frac{\partial E_j}{\partial t}$; this is the fundamental problem of simultaneous actions in decentralized control.
One popular approach is to use a consensus protocol to estimate $\frac{\partial E}{\partial t}$, e.g., ADMM \cite{Summers2012DistributedMultipliers}, which has shown some success in the flocking literature \cite{Lyu2019MultivehicleControl}.
Alternatively, it is possible for the agents to iteratively generate trajectories to converge on a locally optimal control strategy, \cite{Zhan2013FlockingMeasurements,Yuan2017OutdoorControl}.
To minimize the iterative and communication cost, we instead assume that for UAV $i$, the signal $v_j(t)$ is a constant for all $j\in\mathcal{A}\setminus\{i\}$.
As the UAVs form a V shape, their speed will eventually converge to the shape of the formation.
Thus, we expect the error that arises from the constant-speed assumption to asymptotically converge to zero \cite{mahbub2022ACC}.
We demonstrate that the system, using this update scheme, approaches a steady V formation in the following subsections.\footnote{Videos of the experiments and simulation files are available at \url{https://sites.google.com/view/ud-ids-lab/anseroids}.}
Note that $\rho=0$ for all of our results, and this demonstrates that the model mismatch introduced by our scheme is sufficient to ensure that the trivial control policy $v_i = v_j$ for all $i,j\in\mathcal{A}$ does not emerge unless $E_i \approx 0$ or $E_j \approx 0$.

\subsection{Matlab Simulation} \label{sec:matlab}

To demonstrate the behavior of our control algorithm, we apply our proposed constraint-driven control scheme to a system of $N=2$ UAVs Using data for the RQ-11 Raven
\cite{AeroVironment2021RavenRQ-11B}.
The Raven weighs $18.7$ N, with a $1.4$ m wing span and a nominal speed of $12$ m/s.
Approximating the density of air as $\rho = 1.2$ kg/m$^3$ yields the circulation from the Kutta-Joukowski theorem \cite{Anderson2017FundamentalsAerodynamics},
\begin{equation}
    \Gamma = \frac{L}{2b \rho u} = \frac{18.7 \text{ N}}{1.4\text{ m } 1.2\text{ kg/m}^3~ 9 \text{ m/s}} = 1.24 \text{ m}^2\text{/s}.
\end{equation}
We expect $r^*$ to be $\approx 5$\% of the span length (e.g., \cite{Yang2016LoveControl} uses 5.4\%); thus we select $r^* = 0.054$ m.
This implies,
\begin{equation}
    \Omega = \frac{\Gamma}{2\pi(r^*)^2} \approx 70 \text{ s}^{-2}.
\end{equation}

Next we estimate the drag constants in \eqref{eq:drag}.
Assuming steady flight, the induced drag satisfies \cite{Anderson2017FundamentalsAerodynamics},
\begin{equation}
    C_2 = \frac{L^2}{2 \rho \pi b^2 } \approx 95 \text{ N}\cdot\text{m/s}^2.
\end{equation}
Assuming the cruising speed of $12$ m/s is designed to be near-optimal for the Raven,
\begin{equation}
    C_1 = \frac{C_2}{(v_i^*)^4} = 5\times10^{-3} \text{ N}\cdot{s/m}^2.
\end{equation}
Note that this implies a drag force of $0.72$ N at cruising speed, and a dimensionless drag coefficient of $C_d \approx 0.04$, which is not unreasonable for a wing.

We initialize the UAVs in a line along the spanwise axis, with an initial orientation of $\theta_G = 0$ and a center-to-center spacing of $2 b = 1.4$ m.
We present a sequence of simulations in Fig. \ref{fig:simSnapshot}, which demonstrates how two agents quickly fall into a leader and follower position, and that the following agent moves to and remains at the point where the flocking cost is minimized.

\begin{figure*}[ht]
    \centering
    \includegraphics[width=0.32\linewidth]{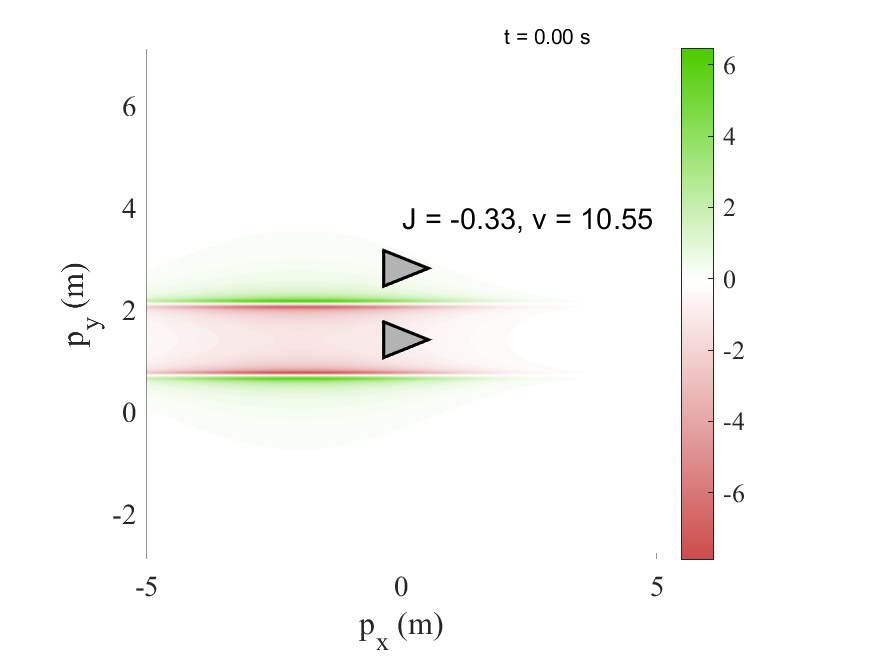}
    \includegraphics[width=0.32\linewidth]{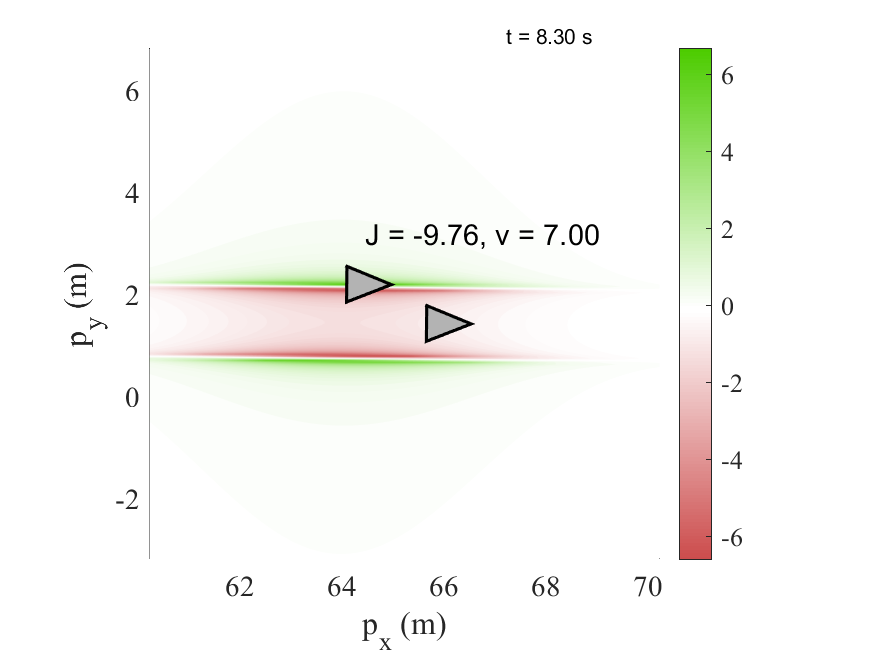}
    \includegraphics[width=0.32\linewidth]{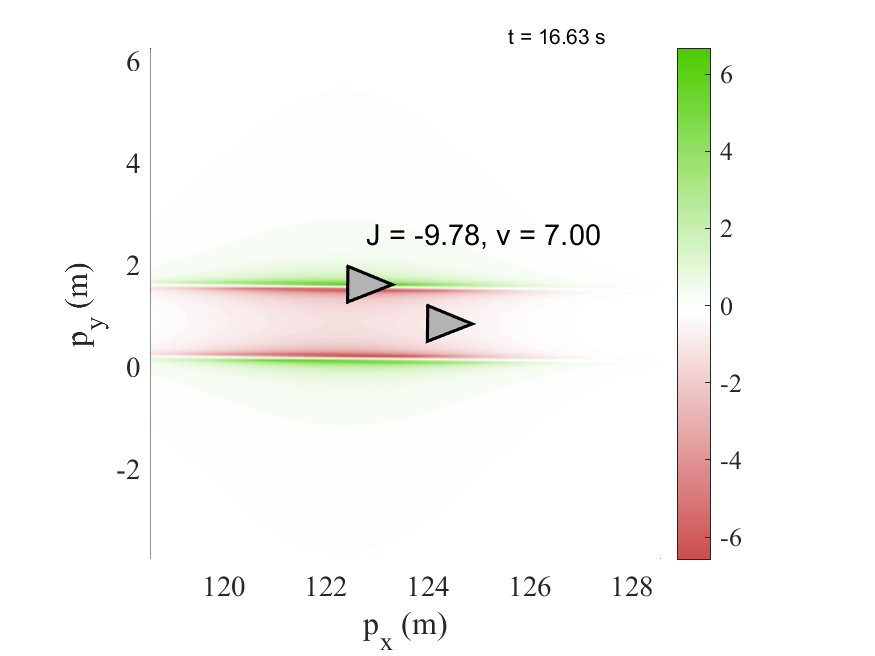}
    \caption{A sequence of simulation snapshots over $20$ seconds for $N=2$ UAVs initialized in a line.
    The annotation shows the instantaneous cost and speed of the top (following) UAV; the contours show the wake of the bottom (lead) UAV.}
    \label{fig:simSnapshot}
\end{figure*}

To quantify the impact of the V formation on the pair of UAVs, we selected $\kappa = 0.25$ and calculated the cost functional $E_i$ for each UAV $i\in\mathcal{A}$.
We present the total cost, the maximum cost, minimum cost, and terminal cost for each UAV in Table \ref{tab:cost}.
Note that a cost of zero corresponds to the agent flying in isolation; this demonstrates that the resulting formation yields a significant energy savings for each UAV.
The instantaneous cost is presented in Fig. \ref{fig:costRange}, which shows the cost experienced the agents at each time-instant of the simulation.
Finally, we can see that the cost remains relatively constant for the last $10$ seconds of the simulation, implying the UAVs have reached a steady-state configuration.

\begin{table}[ht]
    \centering
    \begin{tabular}{rcccc}
        &  Front UAV & Rear UAV  \\\toprule
        Total Cost & -35.2 & -155.2  \\
        Maximum Cost & -0.3 & -0.3  \\
        Minimum Cost & -2.1 & -9.8
    \end{tabular}
    \caption{The impact of our anseroid controller on the cost of each UAV over the entire simulation; a value of $0$ is equivalent to flying in isolation.}
    \label{tab:cost}
\end{table}

\begin{figure}[ht]
    \centering
    \includegraphics[width=\linewidth]{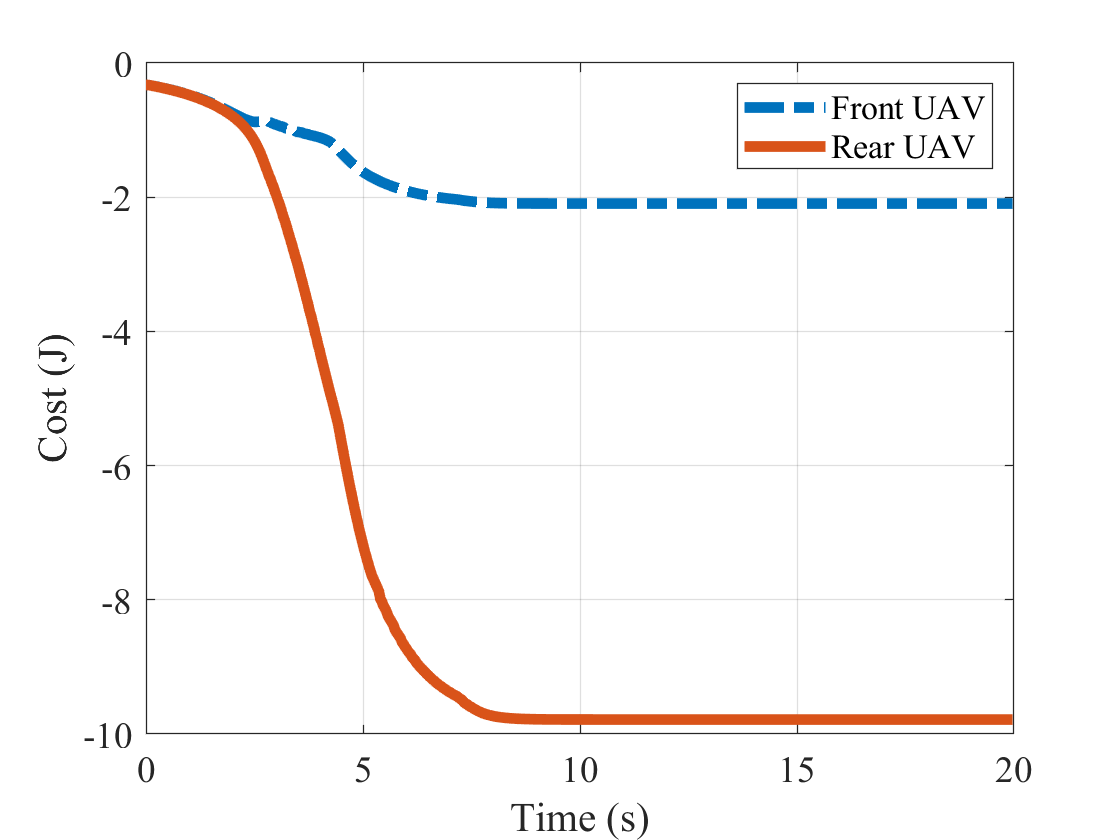}
    \caption{Cost experienced by the front and rear UAVs over the duration of the Matlab simulation. A cost of zero corresponds to flying in isolation.}
    \label{fig:costRange}
\end{figure}

\subsection{Crazy Swarm} \label{sec:experiment}

As a next step toward real-time control of physical UAVs, we developed a decentralized flocking controller using the Crazyswarm library \cite{crazyswarm}.
Crazyswarm is a software library for controlling fleets of Crazyfly quadrotors, and this required us to simulate the aerodynamics of the UAVs to determine the control actions.
Due to the smaller size of the UAVs and the experimental workspace, we scaled the aerodynmic parameters i.e., $\Gamma, \Omega, L$, by $100$ relative to those presented in Section \ref{sec:matlab}, and we present the remaining parameters in Table \ref{tab:parameters}.
To simulate the aerodynamics in real time, we solved the integral \eqref{eq:force} analytically and set $\kappa = 0$.
The simulation result for $N=11$ UAVs is presented in Fig. \ref{fig:crazyswarmSim}.
The front UAV travels almost $35$ m over approximately $60s$, and the formation begins to emerge between $10$ and $15$ m.
This demonstrates that the emergence of a V formation from our controller is independent of the software platform used to implement our algorithm.

\begin{table}[h]
    \centering
    \begin{tabular}{ccccccc}
        $v_{\min}$  & $v_{\max}$  & $w_{\min}$ & $w_{\max}$ & b      & $C_1$ & $C_2$ \\ \toprule
        0.01 m/s    &     2 m/s   &  -1 rad/s  &    1 rad/s & 0.2 m  &    160000   &   5000
    \end{tabular}
    \caption{Parameters used for the Crazyswarm simulation and experiments.}
    \label{tab:parameters}
\end{table}

\begin{figure}[ht]
    \centering
    \includegraphics[width=\linewidth]{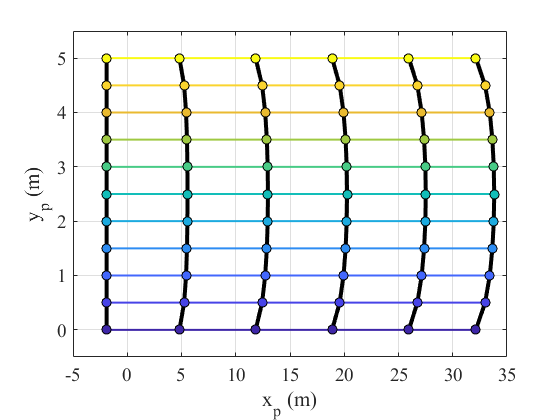}
    \caption{Simulation for $N=11$ UAVs over $60$ seconds in CrazySwarm. The agents are moving from left to right; colored dots and lines correspond to individual UAVs, while the black lines shows the flock shape at different time instants.}
    \label{fig:crazyswarmSim}
\end{figure}

We deployed the same code on a fleet of $N=5$ CrazyFlie quadrotors in our experimental testbed \cite{chalaki2021CSM}.
Quadrotors do not have the same aerodynamic interactions as fixed-wing vehicles, and as a consequence we simulated the upwash field $W_i$ for each $i\in\mathcal{A}$ as part of the control loop.
To generate the control actions, we solved Problem \ref{prb:optCtrl} for each UAV using the SciPy optimization package\footnote{\url{ https://docs.scipy.org/doc/scipy/reference/optimize.html}}.
If the resulting solution violates any constraints, we subsequently solve the relaxed Problem \ref{prb:relaxed} to determine the control input.
Calculating the optimal control action of all five agents took an average of $11$ ms, with a maximum time of $22$ ms, on our centralized computer (i7-6950X CPU @ 3.00 GHz); furthermore, state updates were performed using a VICON motion capture system operating at 100 Hz.

Snapshots of the experiment are presented in Fig. \ref{fig:experiment}; this demonstrates that our control algorithm is real-time implementable and leads to an emergent V formation under the strict time and space constraints present in a laboratory environment.
The agents are initialized in an approximately vertical line and spaced $0.5$ m apart.
The agents at the edge of the formation quickly fall back, and a V formation can be seen to emerge when the front agent has moved only $0.5$ m.
The structure of the  V formation emerges around $2$ m as the agents approach the boundary of our testbed.

\begin{figure*}[ht]
    \centering
    \includegraphics[width=0.32\linewidth]{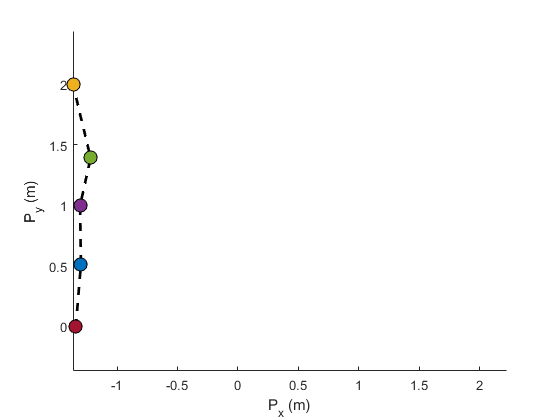}
    \includegraphics[width=0.32\linewidth]{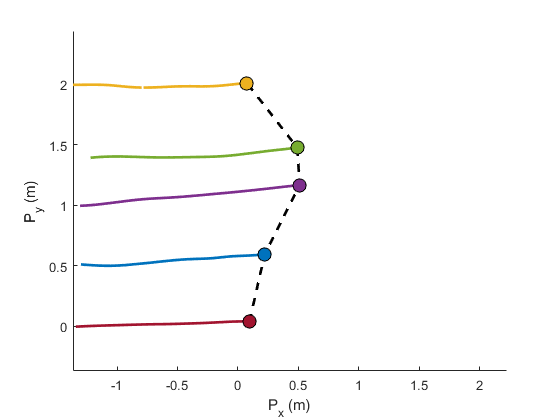}
    \includegraphics[width=0.32\linewidth]{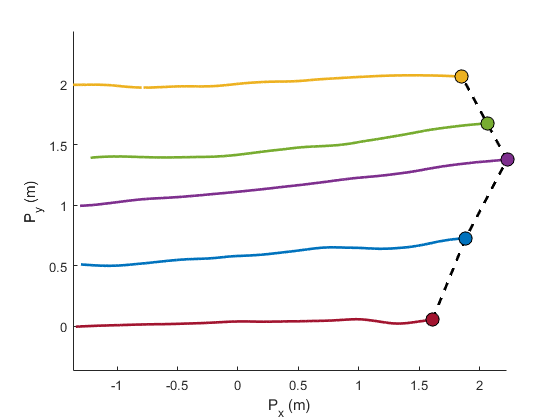}
    \caption{Emergent V formation for $N=5$ Crazyflie drones using the anseroid controller in CrazySwarm over $t=8$ s. The agent positions are shown at the initial time (left), half way through the experiment (center), and the end of the experiment (right).}
    \label{fig:experiment}
\end{figure*}

\section{Conclusions} \label{sec:conclusion}

In this article we presented, for the first time, an anseroid behavior that generates emergent V formations for flocks of UAVs.
First, we demonstrated that a greedy energy-minimizing approach is unstable and cannot lead to a stable V formation, and we proposed a constraint-driven approach where agents maximize their upwash benefits.
We proved for a simplified system that this energy-optimizing control leads to emergent flocking, and that the agents are driven to equlibria that resemble a V or echelon formation.
Finally, we demonstrated the performance of our approach through matlab and python simulations, as well as a real-time experiment with five crazyflie drones.

The results of this article present several compelling directions for future work.
First, analyzing the stability of the V formation to disturbances--including the addition and removal of agents, as well as noise, errors and delays--is critical to understand real-world performance of our proposed control algorithm.
Determining the parameter ranges for the wingspan, aerodynamic properties, weight, and energy factor $\kappa$ that make flocking feasible, i.e., satisfy Theorem \ref{thm:equilibrium}, is another interesting research direction.
Relaxing Assumption \ref{smp:main} using a computational fluid dynamics model may yield interesting insights, as could experiments with real fixed-wing aircraft to infer the upwash force and tendancy to troll from the low-level flight controller.
Finally, exploring leader-switching strategies, and incorporating heterogeneity--through payload mass or fuel quantity--is a compelling research direction for more realistic multi-agent and multi-target missions.

\bibliography{mendeley,IDS_Pubs}

\end{document}